\documentclass{tlp}
\usepackage{latexsym}
\usepackage{amstext,amssymb}

\newcommand{\revise}{\ast}

\newtheorem{definition}{Definition}[section]

\newtheorem{theorem}{Theorem}[section]

\newtheorem{example}{Example}[section]

\newcommand{\la}{\leftarrow}

\newcommand{\nafo}[0]{\mathit{not}}
\newcommand{\naf}[0]{\nafo\;}
\newcommand{\head}[1]{\mathit{head}(#1)}
\newcommand{\pbody}[1]{\mathit{body}^{+}(#1)}
\newcommand{\nbody}[1]{\mathit{body}^{-}(#1)}

\newcommand{\body}[1]{\mathit{body}(#1)}
\newcommand{\atom}[1]{\ensuremath{\mathit{atom}(#1)}}
\newcommand{\Lit}{\ensuremath{\mathcal{L}}} 

\newcommand{\Rem}[2]{\ensuremath{#1\!\downarrow\!#2}}

\newcommand{\cmin}{|\!\ominus\!|^\mathit{min}}

\newcommand{\AS}[1]{\ensuremath{\mathit{AS}(#1)}}
\newcommand{\ThreeAS}[1]{\ensuremath{\mathit{3AS}(#1)}}
\newcommand{\Mod}[1]{\ensuremath{\mathit{Mod}(#1)}}
\newcommand{\CnO}[0]{\ensuremath{\mathit{Cn}}}
\newcommand{\Cn}[1]{\ensuremath{\CnO(#1)}}
\newcommand{\reduct}[2]{\ensuremath{#1^{#2}}}

\title[Logic Program Revision]
{A Program-Level Approach to Revising Logic Programs under the Answer Set
Semantics}

\author[James P.\ Delgrande]
  {%
  JAMES P.\ DELGRANDE\\
  School of Computing Science\\
  Simon Fraser University\\
  Burnaby, B.C.\\
  Canada  V5A 1S6\\
  \email{jim@cs.sfu.ca}
}

\begin{document}

\maketitle

\begin{abstract}
An approach to the \emph{revision} of logic programs under the answer set
semantics is presented.
For programs $P$ and $Q$, the goal is to determine the answer sets that
correspond to the revision of $P$ by $Q$, denoted $P * Q$.
A fundamental principle of classical (AGM) revision, and the one that guides
the approach here, is the \emph{success postulate}.
In AGM revision, this stipulates that $\alpha \in K \revise \alpha$.
By analogy with the success postulate, for programs $P$ and $Q$, this means
that the answer sets of $Q$ will in some sense be contained in those of
$P * Q$.
The essential idea is that for $P * Q$, a three-valued answer set for $Q$,
consisting of positive and negative literals, is first determined.
The positive literals constitute a regular answer set, while the negated
literals make up a minimal set of naf literals required to produce the answer
set from $Q$.
These literals are propagated to the program $P$, along with those rules of
$Q$ that are not \emph{decided} by these literals.
The approach differs from work in \emph{update logic programs} in two main
respects.
First, we ensure that the revising logic program has higher priority, and so
we satisfy the success postulate;
second, for the preference implicit in a revision $P * Q$, the program $Q$
as a whole takes precedence over $P$, unlike update logic programs, since
answer sets of $Q$ are propagated to $P$.
We show that a core group of the AGM postulates are satisfied, as are the
postulates that have been proposed for update logic programs.
\end{abstract}

\begin{keywords}
Answer set programming, belief change, belief revision
\end{keywords}

\section{Introduction}
\label{sec:introduction}

\emph{Answer set programming} (ASP) \cite{baral02a} has proven to be
well-suited to problems in knowledge representation and reasoning (KR).
The advent of efficient provers \cite{SimonsEtAl02,LeoneEtAl06,gekanesc07a} has
led to the successful application of ASP in both KR and constraint
satisfaction problems.
However, an important consideration is that in any nontrivial domain, an
agent's knowledge of the domain will almost certainly be incomplete or
inaccurate, or it may become out of date as the domain evolves.
Thus, over time an agent will need to adjust its knowledge after receiving
new information concerning the domain.

In ASP there has been a substantial effort in developing approaches to
updating a knowledge base, where a knowledge base is expressed as a logic
program under the answer set semantics.
In general, one is given a sequence of logic programs $(P_1, \dots, P_n)$
where informally rules in $P_i$ in some fashion or other take precedence over
rules in $P_j$ for $j < i$.
However, it isn't clear that such approaches capture a notion of
\emph{revision} or \emph{update} of logic programs, so
much as they capture a notion of \emph{priority} or \emph{preference} between
rules in a program.
Thus such approaches generally fail to satisfy properties that would be
expected to hold for revision in classical logic.  
Part of the reason is that revision appears to be intrinsically more
difficult in a nonmonotonic setting (such as in ASP) than in a monotonic one,
such as in classical logic.
We also suggest that part of the problem is that extant approaches
enforce a notion of priority at the level of the individual rule;
instead we propose that the notion of priority given in a revision is a
\emph{program level} notion, in that for a revision specified as $P_1 * P_2$,
program $P_2$ taken as a whole has priority over $P_1$.

In this paper, an approach to the \emph{revision} of logic programs is
presented.
A major goal is to investigate the extent to which AGM-style revision is
compatible with an extended logic programming framework.
A logic program is regarded as a representation of an agent's epistemic state,
while the corresponding answer sets are taken as representing the agent's
beliefs.
The approach describes revision, in that the postulate of \emph{success} is
adhered to;
the idea is that for a revision of $P_1$ by $P_2$, beliefs (viz.\ elements
of an answer set) given by $P_2$ should overrule those in $P_1$.
This is carried out by first determining 3-valued answer sets of $P_2$.
Each such answer set is a pair $(X^+,X^-)$, where $X^+$ is a regular answer
set of $P_2$, and $X^-$ is a minimum set of negation as failure literals
necessary to produce the answer set $X^+$.
The information in each such 3-valued answer set, together with the rules of
$P_2$ not used in the definition of the answer set, and along with the
program $P_1$, is used to define an answer set (or answer sets) of $P_1 * P_2$.

The assumption of \emph{success} leads to an approach with a different
emphasis from previous approaches.
In particular, for the revision of $P_1$ by $P_2$, the \emph{program} $P_2$
is treated as having higher priority than the program $P_1$;
this is in contrast with previous work, wherein the \emph{rules} in $P_2$
are treated as having higher priority than (some or all of) the rules in
$P_1$.
We suggest that this distinction separates approaches addressing
\emph{priorities} in logic programs from \emph{revision}.
As well, it leads to an approach with better properties than earlier work.
For example, the approach is syntax independent, in that if two programs are
strongly equivalent, then they behave the same with respect to revision.
As well, a prototype has been implemented.

We argue that this approach is an appropriate interpretation for a notion of
\emph{revision} in logic programs.
Furthermore, the approach may be applied in cases where a problem is expressed
as a sequence of NP-complete problems;
for example, it allows the natural specification of a problem in which a
3-colouring of a graph is to be found, followed by a Hamiltonian path among
the yellow vertices.
We discuss these issues in more detail in the next section, after formal
preliminaries have been presented.
After this, intuitions are given, and the following section presents the
formal details and properties of the approach.
We conclude with a discussion that includes the applicability of the approach
and future work.
Proofs are for the most part straightforward, and are abbreviated due to
space constraints.

\section{Background}
\label{sec:background}

\subsection{Formal Preliminaries}

Our language is built from a finite set of atoms
${\cal P} = \{a, b, \dots \}$.
A literal is an atom $a$ or its negation $\neg a$;
\Lit\ is the set of literals.
For a set $X$ of literals, $not(X) = \{ not \, a \mid a \in X\}$.
For a literal $l$, $\naf l$ is sometimes referred to as a naf (negation as
failure) literal or \emph{weakly negated} literal.
For $l \in \Lit$, $atom(l)$ is the atom corresponding to $l$;
for a set $X$ of literals, $atom(X) = \{ atom(l) \mid l \in X\}$.
A rule $r$ is of the form
\begin{equation}
\label{eqn:rule}
L_0 \leftarrow L_1, \dots, L_n, \naf L_{n+1}, \dots, \naf L_m
\end{equation}
where $L_0$, \dots, $L_m$ are literals and $0 \leq n \leq m$.
If $n=m$ then $r$ is \emph{positive}.
If $m=0$, then $r$ is called a \emph{fact}.
We also allow the situation where $L_0$ is absent, in which case 
we denote the head by $\bot$;
and $r$ is called a \emph{constraint}.
The literal $L_0$ is called the \emph{head} of $r$, and the set
\(
\{L_1,\dots,L_n,$ $\naf L_{n+1},\dots,\naf L_m\}
\)
is the \emph{body} of $r$.
We use $\head{r}$ to denote the head of rule $r$, and $\body{r}$ to
denote the body of~$r$.
Furthermore,
\(
\pbody{r}
=
\{L_1,\dots, L_n\}
\)
and
\(
\nbody{r}
=
\{L_{n+1},\dots, L_m\}.
\)
An (\emph{extended}\/) \emph{logic program}, or simply a \emph{program},
is a finite set of rules.

A set of literals $X$ is \emph{consistent} if it does not contain a
complementary pair $a$, $\neg a$ of literals and does not contain $\bot$.
We say that $X$ is \emph{logically closed} iff it is either consistent or
equals $\mathcal{L}$.
The smallest set of literals being both logically closed and closed under a
set $P$ of positive rules is denoted by $\Cn{P}$.
The \emph{reduct}, $\reduct{P}{X}$, of $P$ relative to a set $X$ of
literals is
defined by
\(
\reduct{P}{X}
= \{\head{r}\leftarrow\pbody{r}\mid r\in P,\ \nbody{r}\cap X=\emptyset\}
\)
\cite{GelfondLifschitz90}.
A set $X$ of literals is an \emph{answer set} of a logic program $P$ if
$\Cn{\reduct{P}{X}}=X$.
The set of answer sets of program $P$ is denoted by $\AS{P}$.
A program $P$ is consistent just if it has an answer set not equal to
${\cal L}$.
Thus a program with no answer sets is also counted as inconsistent.
For example, the program
\(
P = \{
a\la,\quad
b \la a,\naf{c},\quad
c\la \naf{b}
\}
\)
has answer sets $\AS{P}=\{\{a,b\},\{a,c\}\}$.

Two programs $P_1$ and $P_2$ are \emph{equivalent}, written
\(
P_1 \equiv P_2
\),
if both programs have the same answer sets.
Two programs are strongly equivalent \cite{LifschitzPearceValverde01},
written $P_1 \equiv_s P_2$, just if 
\(
P_1 \cup P_3 \equiv P_2 \cup P_3
\)
for every logic program $P_3$.

\subsection{Belief revision}
\label{sec:revision}

\emph{Belief revision} is the area of KR that is concerned with how an agent
may incorporate new information about a domain into its knowledge base.
In belief revision, a formula $\alpha$ is to be incorporated into the agent's
set of beliefs $K$, so that the resulting knowledge base is consistent when
$\alpha$ is.  Since $\alpha$ may be inconsistent with $K$, revision may also
necessitate the removal of beliefs from $K$ in order to retain consistency.
By a principle of \emph{informational economy}, as many beliefs as
possible are retained from $K$.
A common approach in addressing belief revision is to provide a set of
\emph{rationality postulates} for belief change functions.
The \emph{AGM approach} \cite{agm85,Gardenfors88} provides the best-known set
of such postulates.
An agent's beliefs are modelled by a set of sentences, called a
\emph{belief set}, closed under the logical consequence operator of a logic
that includes classical propositional logic.

Subsequently, various researchers have argued that it is more appropriate to
consider \emph{epistemic states} as objects of revision.
An epistemic state $K$ effectively includes information regarding how the
revision function itself changes following a revision.
The belief set corresponding to epistemic state $K$ is denoted $Bel(K)$. 
Formally, a revision operator $*$ maps an epistemic state $K$ and new
information $\alpha$ to a revised epistemic state $K *\alpha$. 
For set of formulas $\Psi$, define $\Psi + \alpha$ as $\Cn{\Psi\cup\{\alpha\}}$.
Then, the basic AGM postulates for revision can be given as follows:
\begin{description}
\item[($K*1$)]
$Bel(K*{\alpha}) = \Cn{Bel(K*{\alpha})}$

\item[($K*2$)]
$\alpha \in Bel(K*{\alpha})$

\item[($K*3$)]
$Bel(K*{\alpha}) \subseteq Bel(K)+{\alpha}$

\item[($K*4$)]
If
\(
\neg \alpha \notin Bel(K)
\)
then
$Bel(K)+{\alpha} \subseteq Bel(K*{\alpha})$

\item[($K*5$)]
$Bel(K*{\alpha})$
is inconsistent, only if
$\vdash \neg \alpha$

\item[($K*6$)]
If
$\alpha \equiv \psi$ then
$Bel(K*{\alpha}) = Bel(K*{\psi})$
\end{description}
Thus, the result of revising $K$ by $\alpha$ is an epistemic state in which
$\alpha$ is believed in the corresponding belief set ($K*1$, $K*2$);
whenever the result is consistent, the revised belief set consists of the
expansion of $Bel(K)$ by $\alpha$ ($K*3$, $K*4$);
the only time that $Bel(K)$ is inconsistent is when $\alpha$ is inconsistent
($K*5$);
and revision is independent of the syntactic form of the formula for
revision ($K*6$).
(As well, there are two \emph{extended postulates} ($K*7$) and ($K*8$) that
extend ($K*3$) and ($K*4$) to relating revision by a conjunction and the
individual conjuncts.
Like ($K*3$) and ($K*4$) they are are not appropriate in a nonmonotonic
framework (below) and we do not consider them further.)

Belief revision is usually expressed with respect to an underlying logic
governed by a Tarskian consequence operator.
It can be observed that two of the postulates, ($K*3$) and ($K*4$), are
inappropriate in a system governed by a notion of nonmonotonic consequence. 
As an example, consider where the agent believes that a particular
individual is a bird and that it can fly.
If it is subsequently learned that the bird was a penguin, the agent would
also modify its knowledge base so that it believed that the individual did
not fly.
This example then violates both ($K*3$) and ($K*4$).
Note that we can't circumvent this counterexample by simply excluding
states of affairs where there are no flying penguins, since we would want to
allow the \emph{possibility} that a penguin (perhaps an extremely fit penguin)
flies, even though penguins, by default, do not fly.
In consequence we focus on postulates ($K*1$), ($K*2$), ($K*5$) and ($K*6$),
which we refer to as the \emph{core} AGM postulates.

Note that nonmonotonic formalisms can nonetheless be treated from the
standpoint of classical (AGM) revision;
the issue is to express revision in terms of a monotonic foundation.
Thus \cite{desctowo08} addresses revision in ASP from the standpoint of the
\emph{SE models} of a program.
This is in contrast to the work here, and previous work, which addresses
belief change at the level of a logic program rather than with respect to
the underlying models.

Various revision operators have been defined;
see \cite{Satoh88,Williams95,DelgrandeSchaub03} for representative
approaches.
Perhaps the best known approach is the \emph{Dalal revision operator}
\cite{Dalal88}.
This operator is defined as follows.
Let $\psi$, $\mu$ be formulas of propositional logic, and let $\ominus$ be
the symmetric difference of two sets.
Then the Dalal revision $\psi *_d \mu$ is characterised by those models of
$\mu$ that are closest to models of $\psi$, where the notion of closeness is
given by the Hamming distance between interpretations.
Formally, $\psi *_d \mu$, is defined as follows.
For formulas $\alpha$ and $\beta$, define:
\[
\cmin(\alpha, \beta)
\;\;\stackrel{\text{\tiny def}}{=}\;\;
\mathit{min}_\leq(
\{
|w \ominus w'| \mid
w \in \Mod{\alpha}, w' \in \Mod{\beta}
\} ).
\]
Then, \Mod{\psi *_d \mu} is given as
\[
\{
w \in \Mod{\mu}
\mid
\exists w' {\in} \Mod{\psi}
\mbox{ s.t. }
|w \ominus w'| = \cmin(\psi, \mu)
\}.
\]

\subsection{Logic Program Updates}
\label{sec:LPUpdates}

Previous work in ASP that addresses an agent's evolving knowledge base has
generally been termed \emph{logic program update} or
\emph{update logic programs}.
In such approaches one begins with an \emph{ordered logic program},
comprised of a sequence of logic programs $(P_1, \dots, P_n)$.
Rules in higher-ranked sets are, in some fashion or another, preferred over
those in lower-ranked sets.
Commonly this is implemented by using the ordering on rules to adjudicate
which of two rules should be applied when both are applicable and their
respective heads conflict; see for example
\cite{inou-saka-99,AlferesEtAl00,Eiteretal02}.
Alternatively, other approaches use the ordering to ``filter'' rules, e.g.\
\cite{DBLP:conf/ecai/ZhangF98}.
Hence, in one fashion or another, some rules are selected over others, and
these selected rules are used to determine the resulting answer sets.

A major stream of research in ASP has addressed \emph{prioritised} or
\emph{preference} logic programs, where a prioritised logic program is a
pair $(P,<)$ in which $<$ is an ordering over rules in program $P$.
The intuition is that some rules take precedence over (or override or are
more important than) other rules.
Syntactically, the \emph{form} of an update logic program, given as a total
order on programs, is of course an instance of a prioritised logic program.
We suggest that an update logic program is in fact best regarded as a
prioritised logic program.
(Indeed some approaches to updating logic programs are defined in terms of a
prioritised logic program \cite{foo-zhan-97a,DBLP:conf/ecai/ZhangF98}.)
This is most clearly seen in those approaches where the focus is on rules
whose heads conflict.
Thus, for example in the approach of \cite{Eiteretal02}, preferences come
into play only between two rules when the head of one is the complementary
literal of the other.
The following example, due to Patrick Kr\"umpelmann, is illuminating:
\[
P_1 = \{ b \la \}, \quad
P_2 = \{ \neg a \la b \}, \quad
P_3 = \{ a \la \}
\]
In update logic programs, the rule in $P_2$ is dropped, since its head
conflicts with the higher-ranked rule in $P_3$.
Yet dropping the rule in $P_1$ yields a consistent result.
Moreover,
since it is the lowest-ranked rule, arguably it \emph{should} be disregarded.

We can summarise the preceding by suggesting that previous work is essentially
based at the \emph{rule level}, in that higher-ranked rules preempt
lower-ranked rules.
In contrast, the approach here is based at the \emph{program level};
that is for a revision $P_1 * P_2$, the program $P_2$ is considered
\emph{as a whole} to have priority over $P_1$, in that program $P_2$'s
answer sets in a sense have priority over those of $P_1$.
This is effected in a revision $P_1 * P_2$ by first determining answer sets
of $P_2$, and then augmenting, as appropriate, these answer sets with
additional information via $P_1$.
(There is more to it than this, as described in the next section.
For example, rules in $P_1$ may result in other rules in $P_2$ being applied.
However the essential point remains, that the answer sets of $P_1$ are first
determined, and then subsequently augmented.)
Arguably this is the appropriate level of granularity for revision:
If an agent learns new information given in a program $P$, it is the
program \emph{as a whole} that comprises the agent's (new) knowledge.
That is, a rule $r \in P$ isn't an isolated piece of knowledge, but rather,
given possible negation as failure literals in $\body{r}$, the potential
instantiation of $r$ depends non-locally on the entire program $P$.

\cite{Eiteretal02} suggests a number of alternative postulates that may be
considered for update program updates.
For our use, they are given as follows:
\begin{tabbing}
Initialisation:\qquad\quad \= $\AS{\emptyset * P} = \AS{P}$. \kill
\\
Initialisation: \> $\AS{\emptyset * P} = \AS{P}$.
\\
Idempotency: \>
$\AS{P * P} = \AS{P}$. 
\\
Tautology: \>
If $\head{r}\in\pbody{r}$, for all $r\in P_2$,
then
\\
\>
\qquad
$\AS{P_1 * P_2} = \AS{P_1}$.
\\
Associativity: \>
$\AS{P_1 * (P_2 * P_3)} = \AS{(P_1 * P_2) * P_3}$.
\\
Absorption: \>
if $\AS{P_2} = \AS{P_3}$ then
\(
\AS{P_1 * P_2 * P_3} = \AS{P_1 * P_2}.
\)
\\
Augmentation: \>
If $\AS{P_2} \subseteq \AS{P_3}$, then
\(
\AS{P_1 * P_2 * P_3} = \AS{P_1 * P_3}.
\)
\\
Disjointness: \>
If $\atom{P_1}\cap\atom{P_2}=\emptyset$,
then
\\
\>
\qquad
\(
\AS{(P_1 \cup P_2) * P_3} = \AS{P_1 * P_3} \cup \AS{P_2 * P_3}.
\)
\\
Parallelism: \>
If $\atom{P_2} \cap \atom{P_3}=\emptyset$,
then
\\
\>
\qquad
\(
\AS{P_1 * (P_2 \cup P_3)} = \AS{P_1 * P_2} \cup \AS{P_1 * P_3}.
\)
\\
Non-Interference: \>
If $\atom{P_2} \cap \atom{P_3}=\emptyset$,
then  
\\
\>
\qquad
\(
\AS{P_1 * P_2 * P_3} = \AS{P_1 * P_3 * P_2}.
\) 
\end{tabbing}
Many of these postulates are elementary and expected, yet most extant
approaches have problems with them. 
In particular, most approaches do not satisfy \emph{tautology} and so for
instance the addition of a rule $p \la p$ may produce different results.
Moreover, those that do satisfy \emph{tautology} most often do so by
specifically addressing this principle.
It seems reasonable to suggest that the reason for this lack of adherence to
basic postulates is that belief change with respect to ASP is a program-level
operation, and not a rule-level operation.

\section{Logic Program Revision: Intuitions}
\label{sec:intuitions}

The overall goal is to come up with an approach to revision in logic
programs (call it \emph{LP revision}) under the answer set semantics, where
the approach adheres insofar as possible to intuitions underlying classical
(AGM) revision.
In essence, a major goal is to examine the extent to which the AGM approach
may be applied with respect to answer set programming.
As described earlier, we take a logic program $P$ as specifying an agent's
\emph{epistemic state}.
The answer sets of $P$, $\AS{P}$, represent the beliefs of the agent, and so
are analogous to a belief set in AGM revision.

A key characteristic of AGM revision, and one that guides the approach here,
is the \emph{success postulate}.
Recall that in the AGM approach, the success postulate stipulates that
$\alpha \in Bel(K * \alpha)$, or in terms of models, that 
$\Mod{Bel(K * \alpha)} \subseteq \Mod{\alpha}$.
Informally, in a revision by $\alpha$, the logical content of $\alpha$ is
retained.
By analogy with the success postulate, for a revision of $P_1$ by $P_2$, the
content of $P_2$ is given by its answer sets, and so in the revision
$P_1 * P_2$, the answer sets of $P_2$ should in some sense be contained in
those of $P_1 * P_2$.
This notion is fundamental; as well, it has very significant ramifications
in an approach to LP revision.

For example, consider the following programs, where we want to determine
$P_1*P_2$:
\begin{example}
\label{ex:1}
\begin{eqnarray*}
P_1 & = & \{b \leftarrow, \;\; c \leftarrow \naf d \} \\
P_2 & = & \{a \leftarrow \naf b\}
\end{eqnarray*}
\end{example}
%
%
By our interpretation of the success postulate, since $\{a\}$ is an answer
set of $P_2$, it should appear in the answer sets of $P_1 * P_2$
(that is, $\{a\}$ should be a subset of some answer set of $P_1 * P_2$).
However, $a$ was derived by the failure of being able to prove $b$ in $P_1$.
Consequently, if the answer sets of $P_2$ are to appear among the answer sets
of $P_1 * P_2$, then the \emph{reasons} for the answer sets of $P_2$ should
also be retained.
Consequently $b$ should not appear in the answer sets of $P_1 * P_2$.
Hence we would want to obtain $\{a,c\}$ as the answer set of $P_1 * P_2$.
This example serves to distinguish the present approach from previous work,
in that in previous work the assertion of a fact overrides an assumption of
negation as failure at any level.
Thus in previous work on update logic programs for the above example one
would obtain the answer set $\{b, c\}$.

So adherence to a success postulate requires that, if a literal $\naf p$ is
used in a higher-ranked set of rules, it should override positive occurrences
in lower-ranked sets.
This also is in keeping with our assertion in the previous section, that in
a revision we consider a program as a whole, and not at the individual rule
level.
However, it might plausibly be objected that often one wants to retain facts
(such as $b$ in $P_1$), and so such facts should obtain in the revision
$P_1 * P_2$.
We suggest instead that in such an instance, such (protected) facts should
in fact be given higher priority.
We return to this point in Section~\ref{sec:discussion}, where we discuss the
notion of a \emph{revision methodology}.

In working towards an answer set for a revision $P_1 * P_2$, we first
determine an answer set for $P_2$.
However, we need to keep track of not just those literals that are (positively)
derivable, but also a set of $not$ literals necessary for the construction of
the answer set.
Consequently, we deal with three-valued answer sets.
Thus for Example~\ref{ex:1}, in considering $P_1$, we need to keep track of the fact
that $a$ was derived in $P_2$ and that moreover $\naf b$ was used in this
derivation, thereby necessitating the blocking of any later deriving of $b$
in lower ranked rule sets.
We write the three valued answer set of $P_2$ in Example~\ref{ex:1} as
$(\{a\}, \{b\})$.
The three value answer set for $P_1 * P_2$ then is $(\{a,c\}, \{b,d\})$;
and the corresponding answer set for $P_1 * P_2$ is $\{a,c\}$.

Consider next a variation of Example~\ref{ex:1} where again we are to determine
answer sets for $P_1 * P_2$:
\begin{example}
\label{ex:1b}
\begin{eqnarray*}
P_1 & = & \{b \leftarrow, \;\; c \leftarrow \naf d \} \\
P_2 & = & \{a \leftarrow \naf b, \;\; a \leftarrow \naf e\}
\end{eqnarray*}
\end{example}
%
The atom $a$ may be obtained by $\naf b$ or $\naf c$ in $P_2$.
By appeal to a principle of \emph{informational economy}
(Section~\ref{sec:background}), in a three valued answer set we retain a
minimum number of $not$ literals sufficient to derive the answer set.
This commitment to a minimum number of $not$ literals in turn means that
there are fewer restrictions when next considering $P_1$.
In the present example, this means that $P_2$ has two three-valued answer
sets:
$(\{a\}, \{b\})$ and $(\{a\}, \{c\})$.
This leads to the three valued answer sets for $P_1 * P_2$:
$(\{a,c\}, \{b\})$ and $(\{a,b\}, \{c\})$; and corresponding answer sets
$\{a,c\}$ and $\{a,b\}$.

Consider finally the programs:
\begin{example}
\label{ex:2}
\begin{eqnarray*}
P_1 & = & \{b \leftarrow \} \\
P_2 & = & \{a \leftarrow b\}
\end{eqnarray*}
\end{example}
%
$P_2$ has answer set $\emptyset$ (and three-value answer set
$(\emptyset,\emptyset)$).
However in next considering $P_1$ in the revision $P_1 * P_2$, one should be
able to use the non-satisfied rule $a \leftarrow b$ and obtain an answer set
$\{a,b\}$ for $P_1 * P_2$.
This is by way of an extended notion of informational economy, in which a
maximal justifiable set of beliefs is desirable.
So rules of $P_2$ that are neither applied nor refuted should nonetheless be
available for later steps in the revision.

These examples have dealt with a single occurrence of revision.
Clearly the process can be iterated to a sequence of programs.
Informally, an answer set for a sequence of programs is determined by
finding 3-valued answer sets for higher-ranked programs, and propagating
these answer sets, along with \emph{undecided} rules, to lower ranked
programs.
Consequently, answer sets are built incrementally, with literals at a higher
level being retained at lower levels.
In the next subsection, for generality, we work with sequences of programs
rather than just pairs.

\section{Logic Program Revision: Approach}

This section describes an approach to LP revision based on the intuitions of
the previous subsection.
Consider by way of analogy, classical AGM revision: 
For a revision $K * \alpha$, the formula $\alpha$ is to be incorporated in $K$;
since $Bel(K) \cup \{\alpha\}$ may well be inconsistent, formulas in
$Bel(K)$ may be dropped in order to obtain a consistent result.
Similarly in a revision of programs $P_1 * P_2$:
we would like the result to be consistent if possible.

In outline, the goal is to determine answer sets for $P_1 * P_2$.
To this end, an answer set $X$ of $P_2$ is determined and it, along with the
rules in $P_2$, say $P_2'$, that do not take part in the definition of $X$,
are propagated to $P_1$.
Since the result should be consistent, we consider maximal subsets of
$P_1$ that are consistent with $P_2' \cup X$ and use these to determine the
resulting answer sets for the revision.
We begin by defining the relevant notion of an answer set with respect to
revision.

By a three-valued interpretation we will mean an ordered pair of sets
$X = (X^+,X^-)$ where $X^+,X^- \subseteq \Lit$ and
$X^+ \cap X^- = \emptyset$.
The intuition is that members of $X^+$ constitute an answer set of some
program, while $X^-$ contains a minimum set of assumptions necessary for the
derivation of $X^+$.
A (canonical) program corresponding to a 3-valued interpretation
$X = (X^+,X^-)$ is given by
\[
Pgm(X) \;=\:
\{ a \la \; \mid a \in X^+ \}
\;\cup\;
\{ \bot \la a \mid a \in X^- \}.
\]
The consequence relation $Cn(.)$ on definite programs is extended to
arbitrary logic programs by the simple expedient of treating a weakly negated
literal $\naf l$ as a new atom.
Thus for example $Cn(\{ a \la,\;\; b \la a,\;\; c \la \naf d \})$ is
$\{a,b\}$.

The next definition extends the notion of reduct to 3-valued
interpretations.
\begin{definition}\label{def:min-reduct}
Let $P$ be a logic program and $X = (X^+, X^-)$ a 3-valued interpretation.

$P^X$, the \emph{min-reduct of $P$ wrt $X$}, is the program obtained from
$P$ by:
\begin{enumerate}
\item
deleting every rule
$r \in P$
where
$body^-(r) \cap X^+ \neq \emptyset$, and

\item
replacing any remaining rule $r \in P$
by

\qquad
\(
head(r) \la body^+(r), not(body^-(r) \setminus X^-).
\)
\end{enumerate}
\end{definition}
Part 1 above is the same as in the standard definition of reduct.
In Part 2, just those naf literals appearing in $X^-$ are deleted from the
bodies of the remaining rules.
The following definition extends the notion of an answer set to 3-valued
answer sets.
\begin{definition}\label{def:min-AS}
Let $P$ be a logic program and $X = (X^+, X^-)$ a 3-valued interpretation.

$X = (X^+, X^-)$ is a
\emph{3-valued answer set for $P$} if
$Cn({P^X}^+) = Cn(P^X) = X^+$ and for any $Y = (X^+,Y^-)$ where
$Y^- \subset X^-$
we have that $Cn(P^{Y}) \neq X^+$.

The set of 3-valued answer sets of program $P$ is denoted $\ThreeAS{P}$.
\end{definition}
Thus for 3-valued answer set $X = (X^+, X^-)$ of $P$, we have that $X^+$ is an
answer set of $P$.
As well, 3-valued answer sets include sufficient negation as failure literals
for the derivation of the answer set.
Thus, for $\{ a \la \naf b, \; a \la \naf c \}$ there
are two 3-valued answer sets $(\{a\},\{b\})$ and $(\{a\},\{c\})$ along with 
answer set $\{a\}$.

As suggested at the start of the section, in a revision $P_1 * P_2$ we need
to isolate a subset of $P_1$ that is consistent with $P_2$. 
We give the necessary definition next.

\begin{definition}
\label{def:remainder}
Let $P_1$, $P_2$ be logic programs.
Define $\Rem{P_1}{P_2}$ by:

\noindent
If $P_2$ is not consistent, then $\Rem{P_1}{P_2} = {\cal L}$.

\noindent
Otherwise:
\begin{eqnarray*}
\Rem{P_1}{P_2}
&=&
\{
P' \cup P_2 \mid
P' \subseteq P_1 \mbox{ and }
P' \cup P_2 \mbox{ is consistent and }
\\
&&
\qquad
\qquad
\mbox{ for $P' \subset P''  \subseteq P_1$, }
P'' \cup P_2 \mbox{ is inconsistent.}
\}
\end{eqnarray*}
\end{definition}

\medskip

Thus for $P \in \Rem{P_1}{P_2}$, $P$ consists of $P_2$ together with a maximal
set of rules from $P_1$ such that $P$ is consistent.

Given a sequence of logic programs $(P_1, \dots, P_n)$, the revision process
can now be informally described as follows:
\begin{enumerate}
\item
Let $X_n \in \ThreeAS{P_n}$; that is, $X_n = (X_n^+,X_n^-)$ is a
3-valued answer set for $P_n$.

\item
In the general case, one has a 3-valued answer set
$X_{i+1} = (X_{i+1}^+,X_{i+1}^-)$ from the revision sequence
$(P_{i+1}, \dots, P_n)$.
A maximal set of rules in $(P_i, \dots, P_n)$ consistent with $X_{i+1}$ is
used to determine a 3-valued answer set $X_i = (X_i^+,X_i^-)$ for the revision
sequence $(P_i, \dots, P_n)$.

\item
A 3-valued answer set $X_1 = (X_1^+,X_1^-)$ for the revision sequence
$(P_1, \dots, P_n)$ then yields the answer set $X_1^+$ for the full
sequence $(P_1, \dots, P_n)$.
\end{enumerate}
With this setting, we can give the main definition for the answer sets
of a revision sequence of programs.
To this end, a revision problem is given by a sequence $(P_1, \dots, P_n)$
of logic programs;
the goal is to determine answer sets of the sequence under the interpretation
that a higher-indexed program, taken as a single entity, takes priority over
a lower-indexed program.
We write the revision sequence in notation closer to that of standard belief
revision, as $P_1 * \dots * P_n$; 
our goal then is to characterise the answer sets of $P_1 * \dots * P_n$.

\begin{definition}\label{def:revision-as}
Let $P=(P_1, \dots, P_n)$ be a sequence of logic programs.

$X$ is an answer set of $P_1 * \dots * P_n$ iff there is a sequence:
\[
((P^r_1, X_1), \dots, (P^r_n, X_n))
\]
such that for $1 \leq i \leq n$, $P^r_i$ is a logic program and $X_i$ is a
3-valued interpretation,

\noindent
and:

\begin{enumerate}
\item
i.)
\quad\ 
\(
P^r_n = P_n
\)
\qquad
and

ii.)
\quad
\(
X_n \mbox{ is a 3-valued answer set for $P_n$.}
\)

\item
for $i < n$:

\qquad
i.)
\quad\ 
\(
P^r_i
\;\in\;
\Rem{P_i}{(P^r_{i+1} \cup Pgm(X_{i+1}))}
\)
\qquad
and

\qquad
ii.)
\quad
\(
X_i \mbox{ is a 3-valued answer set for $P^r_i$.}
\)

\item
$X = X_1^+$.
\end{enumerate}
The set of answer sets of $P_1 * \dots * P_n$ is denoted
$\AS{P_1 * \dots * P_n}$.
\end{definition}

The case of binary revision is of course simpler.
Though redundant, it is instructive, and so we give it next.

\begin{definition}\label{def:binrevision-as}
Let $P_1$, $P_2$ be logic programs.
$X$ is an answer set of $P_1 * P_2$ if:

\begin{enumerate}
\item
there is a 3-valued answer set $X_2$ of $P_2$ and,

\item
for some
\(
P \in \Rem{P_1}{(P_2 \cup Pgm(X_2))},
\)
$X$ is an answer set of $P$.
\end{enumerate}
\end{definition}

\subsection{Examples}
Consider the examples given earlier.
For Example~\ref{ex:1} we have:
\begin{eqnarray*}
P_1 & = & \{b \la, \; c \la \naf d \} \\
P_2 & = & \{a \la \naf b\}
\end{eqnarray*}

\noindent
For $P_1 * P_2$ we obtain:
\[
\begin{array}{ll}
P^r_2 \;=\; P_2,
&
X_2 \;=\; (\{a\},\{b\}) 
\\
P^r_1
\;=\;
\{ c \la \naf d \} \cup \{a \la, \; \bot \la b\},
\qquad
&
X_1 \;=\; (\{a,c\},\{b,d\}).
\end{array}
\]
Thus, $P_2$ has 3-valued answer set $(\{a\},\{b\})$ --
$\{a\}$ is a (standard) answer set for $P_2$, and it depends on, at minimum,
the assumption of $b$ being false by default.
This in turn requires a commitment to the non-truth of $b$ in next considering
$P_1$.
The program $P^r_1$ given by rules of $P_1$ consistent with
$(\{a\},\{b\})$ consists of the single rule $c \la \naf d$ along with an
encoding of the 3-valued interpretation $(\{a\}, \{b\})$.
We obtain the 3-valued answer set $(\{a,c\},\{b,d\})$, with corresponding
answer set $\{a,c\}$.
This shows that a fact, in this case $b$, may be withdrawn without its
negation $\neg b$ being asserted.
We elaborate on this point in the context of an overall revision methodology
in Section \ref{sec:discussion}.

Consider next Example~\ref{ex:1b}:
\begin{eqnarray*}
P_1 & = & \{b \la, \;\; c \la  \} \\
P_2 & = & \{a \la \naf b, \;\; a \la \naf c\}
\end{eqnarray*}
%
$P_2$ has two 3-valued answer sets $(\{a\},\{b\})$ and $(\{a\},\{c\})$.
Again, $\{a\}$ is a (standard) answer set for $P_2$, but it depends on,
at minimum, the assumption of either $b$ or $c$ being false by default.
This in turn requires a commitment to the falsity of one of $b$ or $c$ in
next considering $P_1$.
As a result, the 3-valued answer set $(\{a\},\{b\})$ yields the 
program $P^r_1$ given by $\{ c \la, \; a \la, \; \bot \la b \}$,
while the 3-valued answer set $(\{a\},\{c\})$ yields the program
$\{ b \la,  \; a \la, \; \bot \la c  \}$.
Consequently for the revision we obtain two 3-valued answer sets
$(\{a,c\},\{b\})$ and $(\{a,b\},\{c\})$, with corresponding answer sets
$\{a,c\}$ and $\{a,b\}$.

For Example~\ref{ex:2}, where $P_1 = \{b \la \}$ and $P_2 = \{a \la b\}$,
we obtain:
\[
\begin{array}{ll}
P^r_2 \;=\; P_2,
&
X_2 \;=\; (\emptyset, \emptyset)
\\
P^r_1
\;=\;
\{ b \la \} \cup \{ a \la b \},
\qquad
&
X_1 \;=\; (\{a,b\},\{\}).
\end{array}
\]
Thus there is one 3-valued answer set, $(\{a,b\}, \emptyset)$, with answer set
$\{a,b\}$.

We consider two more small examples to further illustrate the approach.
\begin{eqnarray*}
P_1 & = & \{a \la, \;\; b \la  \} \\
P_2 & = & \{ \bot \la a, b\}
\end{eqnarray*}
For $P_1 * P_2$, there are two 3-valued answer sets,
$(\{a\},\emptyset)$, $(\{b\},\emptyset)$ with corresponding answer sets
$\{a\}$, $\{b\}$.
This is what would be desired:
$P_2$ requires that $a$ and $b$ cannot be simultaneously true, while $P_1$
states that $a$ and $b$ are both true.
In this case, $P_2$ is retained, with a ``maximal'' part of $P_1$ also
held.
Last, consider:
\begin{eqnarray*}
P_1 & = & \{a \la, \;\; d \la b \} \\
P_2 & = & \{b \la \naf a, \;\; c \la \naf \neg a\}
\end{eqnarray*}
In this case, $P_2$ has three-valued answer set $(\{b, c\}, \{a, \neg a\})$;
hence the derivation of $b$ and $c$ relies on the possibility of $a$ being
true, and of $a$ being false.
There is one 3-valued answer set for $P_1 * P_2$,
$(\{b,c,d\}, \{a, \neg a\})$, with answer set $\{b,c,d\}$.

\subsection{Properties}

Section~\ref{sec:revision} noted that four of the basic AGM postulates
are appropriate in a nonmonotonic framework.
With respect to these postulates, we obtain the following:

\begin{theorem}
Let $P_1$, $P_2$, $P_3$ be logic programs.

\begin{description}
\item[$(A * 1)$]
$\AS{P_1 * P_2} \subseteq 2^{\cal L}$.

\item[$(A * 2)$]
If $X \in \AS{P_2}$ then there is
\(
X' \in \AS{P_1 * P_2}
\)
such that $X \subseteq X'$.

\item[$(A * 5a)$]
\(
\AS{P_1 * P_2} = {\cal L}
\)
iff
$\AS{P_2} = {\cal L}$

\item[$(A * 5b)$]
\(
\AS{P_1 * P_2} = \emptyset
\)
iff
$\AS{P_2} = \emptyset$

\item[$(A * 6)$]
If
\(
P_2 \equiv_s P_3
\)
then
\(
\AS{P_1 * P_2} = \AS{P_1 * P_3}.
\)
\end{description}
\end{theorem}

\begin{proof}
$(A * 1)$ is a direct consequence of Definition~\ref{def:binrevision-as}.
$(A * 2)$ follows from the fact that if $X \in \AS{P_2}$ then there is a
three-valued answer set $(X,Y) \in \ThreeAS{P_2}$, and that
$\Rem{P_1}{(P_2 \cup Pgm((X,Y)))}$ by definition has an answer set.
Similarly, for $(A * 5a)$ we have that if
$\AS{P_2} = {\cal L}$ then $\AS{P_1 * P_2} = {\cal L}$, and
if $\AS{P_2} \neq {\cal L}$ then $\AS{P_1 * P_2} \neq {\cal L}$, both
following directly from Definition~\ref{def:remainder}.
$(A * 5b)$ follows analogously.

For $(A * 6)$ we have, in outline:
By assumption we have that $P_2 \equiv_s P_3$, and so for any program $R$,
$P_2 \cup R \equiv_s P_3 \cup R$.
In particular, let $X \in \ThreeAS{P_2}$.
Then $P_2 \cup Pgm(X) \equiv_s P_3 \cup Pgm(X)$.

We next make use of the small result: If $P \equiv_s P'$ then for program $R$
and $S \in \Rem{R}{P}$ we have that there is $S' \in \Rem{R}{P'}$ such that
$\AS{S} = \AS{S'}$.
The proof of this claim is straightforward:
Let $S \in \Rem{R}{P}$.
$S$ is of the form $R' \cup P$ where $R' \subseteq R$.
We have that $R' \cup P \equiv_s R' \cup P'$ since $P \equiv_s P'$.
So for $S' = R' \cup P'$ we have that $S' \in \Rem{R}{P'}$;
as well, $S \equiv_s S'$ and so $\AS{S} = \AS{S'}$.

So, consider again $X \in \ThreeAS{P_2}$.
We have that
$P_2 \cup Pgm(X) \equiv_s P_3 \cup Pgm(X)$
and so, applying the above small result we have that for
$S \in \Rem{P_1}{(P_2 \cup Pgm(X))}$ there is
$S' \in \Rem{P_1}{(P_3 \cup Pgm(X))}$, such that
$\AS{S} = \AS{S'}$, and so if follows that
$\AS{P_1 * P_2} = \AS{P_1 * P_3}$.
\end{proof}

Thus the result of revision is a set of answer sets $(A * 1)$, which is to
say, if an agent's beliefs are given by a set of answer sets corresponding
to potential states of the world, then a revision sequence also yields a set
of such beliefs.
The key property of the approach is given by $(A * 2)$, corresponding to the
success postulate:
in a revision $P_1 * P_2$, beliefs as expressed in $P_2$ override those of
$P_1$.
The two parts of $(A * 5)$ hold by virtue of the fact that in a revision
$P_1 * P_2$, only some consistent (with $P_2$) subset of $P_1$ is used in
the revision.
$(A * 6)$ is a version of independence of syntax.
The postulate fails if $P_2 \equiv_s P_3$ is replaced by $P_2 \equiv_u P_3$
or $\AS{P_2} = \AS{P_3}$:
a counterexample is given by $P_1 = \{ b \la \}$, $P_2 = \{ a \la \naf b\}$,
and $P_3 = \{ a \la \naf c\}$.
Consequently, appropriate versions of the core AGM postulates hold in the
approach.

With regards to the postulates given in Section~\ref{sec:LPUpdates} for logic
program updates, we obtain the following.

\begin{theorem}
Let $P_1$, $P_2$, $P_3$ be logic programs.

Then $P_1$, $P_2$, $P_3$ satisfy \emph{initialisation}, \emph{idempotency},
and \emph{non-interference}.

\medskip

As well, we obtain:

{\bf SAbsorption:}
if $P_2 \equiv_s P_3$ then
\(
\AS{P_1 * P_2 * P_3} = \AS{P_1 * P_2}.
\)
\end{theorem}
These principles are elementary but nonetheless desirable.
As mentioned, most approaches to update logic programs fail to satisfy
\emph{tautology}, and in fact they generally do so when a revised program
is inconsistent.
The present approach also fails to satisfy \emph{tautology}, and in the same
situation, but in this case the lack of this principle is intended:
\emph{tautology} is incompatible with $(A * 5)$ and $(A * 6)$ in the case
that the revised program is inconsistent.
Note that it is straightforward to modify the approach so that
\emph{tautology} is satisfied, at the expense of $(A * 5)$ and $(A * 6)$, by
the simple expedient of modifying Definition~\ref{def:binrevision-as}, Part 2,
to be
\begin{enumerate}
\item[2.]
$X$ is an answer set for $P_1 \cup P_2 \cup Pgm(X_2)$.
\end{enumerate}

The remaining principles listed in Section~\ref{sec:background} can be argued
to be undesirable, with the possible exception of \emph{associativity}.
Not surprisingly, \emph{absorption} fails at the level of answer sets, though
not at the level of strong equivalence (as given by {\bf SAbsorption}).
\emph{Augmentation} would seem to be related to a notion of monotonicity, and
hence is undesirable.
\emph{Disjointness} and \emph{parallelism} both clearly fail.
Arguably both \emph{should} fail;
consider the the case where $P_2 = \emptyset$.
Disjointness in this situation reduces to:
\[
\AS{P_1 * P_3} = \AS{P_1 * P_3} \cup \AS{P_3}
\]
which is clearly undesirable.

\paragraph{Implementation}

A prototype implementation has been written in C \cite{Tasharrofi09}, and
making use of the solver {\tt clasp} \cite{gekanesc07a}.
The prototype is at an early stage of development, and results concerning
benchmarks or scalability have not yet been obtained.

\section{Discussion}
\label{sec:discussion}

This paper has described an approach to logic program revision in which the
focus is on \emph{revision} as understood in the belief revision community.
Consequently, the key \emph{success} postulate is taken seriously.
The intuition is that a logic program represents an agent's epistemic state,
while the answer sets are a representation of the agent's contingent beliefs.
This leads to an approach with quite different properties than other
approaches that have appeared in the literature.
In particular, for a revision $P_1 * P_2$ the program $P_2$ is treated as a
whole as having higher priority than $P_1$, in that answer sets of $P_2$ are
propagated to $P_1$.
This has an important consequence, and requires the use of three-valued
interpretations, in that literals assumed to be false at a higher ranked
program can override literals used as facts in a lower-ranked program.
This is in contrast to logic program update, where one selects rules
to apply, giving preference to rules in $P_2$, and then applies these selected
rules.

Arguably the approach helps cast light on the logic-program-update landscape.
We have suggested that \emph{update logic program} approaches are more
appropriately viewed as dealing with preferences or priorities over rules,
rather than revision or update per se.
The approach at hand seems to fall somewhere between a syntactic approach
and a semantic one, such as \cite{desctowo08}.
The answer sets obtained exhibit reasonable properties -- for example, the
core AGM postulates are satisfied, including syntax independence under strong
equivalence and a success postulate; and the appropriate set of logic program
update postulates are also satisfied.
%
The approach would seem to have some hope for practical implementation.
A prototype implementation is available.
At present it is essentially a ``proof of concept'' and is a subject of
current investigation.

\paragraph{Applicability and Evaluation of the Approach}
As indicated in the introduction, a primary aim of this paper is to
investigate the extent to which classical notions of belief revision may be
applied in revising logic programs under the answer set semantics.
We have claimed that the approach is \emph{intuitive}, in that it reflects
plausible notions concerning revision.
In contrast to related work, it satisfies a suite of desirable formal
properties.
On the other hand, as with belief revision in classical logic, there is not
necessarily a single best approach to logic program revision, but rather
different approaches will be more or less suitable in different areas.
To this end we can consider circumstances in which the present approach would
be applicable.

It would be applicable, clearly, in a problem in which the answer sets of
more recent programs are to hold sway over the answer sets of lower-ranked
programs.
This would be the case with \emph{program refinement} for example, where
elaborations taking care of special cases are incorporated into an existing
program.
As well, the approach can be directly applied to problems involving a sequence
of NP complete problems, for example in a situation where the solution to one
problem feeds as input into a second.
A representative example described earlier involves finding a three colouring
for a graph, and then finding a Hamiltonian path on the vertices of a
specific colour.

These considerations also suggest the following \emph{revision methodology}.
A logic program will most often be comprised of two parts, a
\emph{problem instance} consisting of \emph{facts}, and a 
\emph{problem description} consisting of rules with variables.
Typically the facts will be of highest importance -- thus in the example of
finding a Hamiltonian path among yellow-coloured vertices, one would not
want to revise the underlying graph.
Hence \emph{any} approach to logic program change will want to isolate
portions of the knowledge base (usually the problem instance) as being
incontrovertible.
It is a strength of the present approach that this can already be effected
in an application, since if the facts are given as the first revising program
in a sequence, the success postulate guarantees that the facts will not be
modified.

\paragraph{Future work}
There has been little work on logic program change with respect to
disjunctive logic programs.
A major reason for this is that most approaches to update logic programs
focus on rules with conflicting heads.
In a disjunctive program, it is not immediately clear what it means for two
rules to have conflicting heads.
Consequently it isn't obvious how approaches to update logic programs can
be extended to deal with disjunction in the head.
An obvious extension to the present approach then is to apply it to
disjunctive programs.
The key issue is to extend the definition of a three-valued interpretation
to disjunctive rules.
There seems to be no barrier to such an extension, and this would seem to be
a useful yet (presumably) straightforward extension.

A more difficult problem is to define revision so that the result
of a revision is a single logic program expressed in the language
of the revising programs.
That is, intuitively and ideally, the result of $P_1 * \dots * P_n$ should
be a logic program $P$, expressed in the language
$\atom{P_1} \cup \dots \cup \atom{P_n}$.
(That is, auxiliary atoms
would be expressly forbidden.)
Arguably \emph{this} is the real problem of revision in logic programs, and
there has been no substantial progress in this area due to its difficulty.
The present approach may allow insight into how such an approach to revision
can be carried out.
In fact, it already allows the formation of single logic program that can be
considered as a candidate for the revision $P_1 * \dots * P_n$.
We have the result:
\begin{theorem}
\label{thm:asfromrevis}
Given the assumptions and terms of Definition~\ref{def:revision-as} where
$(X^+_1,X^-_1)$ is a 3-valued answer set of $P_1 * \dots * P_n$ we have
that:
\quad
$X^+_1$ is an answer set of
\(
P^r_1 \cup P^r_2 \cup \dots \cup P^r_n.
\)
\end{theorem}
Note that $P^r_n = P_n$ and so the highest-ranked program is retained,
along with other rules from remaining programs that aren't excluded by
higher-ranked programs.
Hence, we can use the approach to obtain a single program from a sequence of
programs.
However, this is a weak result, since the above result guarantees only that a
single answer set of $P_n$ will be retained in the program resulting from the
revision $P_1 * \dots * P_n$.
Hence one obtains a very credulous revision operator, in that many potential
answer sets are ruled out.
It may be that the set of programs obtained according to
Theorem~\ref{thm:asfromrevis} can be \emph{merged} into a single program, but
this task in turn requires a nuanced approach to merging logic programs.


\bibliographystyle{acmtrans}
\bibliography{ai,lit,local}

\begin{thebibliography}{}

\bibitem[\protect\citeauthoryear{Alchourr\'{o}n, G{\"{a}}rdenfors, and
  Makinson}{Alchourr\'{o}n et~al\mbox{.}}{1985}]{agm85}
{\sc Alchourr\'{o}n, C.}, {\sc G{\"{a}}rdenfors, P.}, {\sc and} {\sc Makinson,
  D.} 1985.
\newblock On the logic of theory change: Partial meet functions for contraction
  and revision.
\newblock {\em Journal of Symbolic Logic\/}~{\em 50,\/}~2, 510--530.

\bibitem[\protect\citeauthoryear{Alferes, Leite, Pereira, Przymusinska, and
  Przymusinski}{Alferes et~al\mbox{.}}{2000}]{AlferesEtAl00}
{\sc Alferes, J.~J.}, {\sc Leite, J.~A.}, {\sc Pereira, L.~M.}, {\sc
  Przymusinska, H.}, {\sc and} {\sc Przymusinski, T.~C.} 2000.
\newblock Dynamic updates of non-monotonic knowledge bases.
\newblock {\em Journal of Logic Programming\/}~{\em 45,\/}~1-3, 43--70.

\bibitem[\protect\citeauthoryear{Baral}{Baral}{2003}]{baral02a}
{\sc Baral, C.} 2003.
\newblock {\em Knowledge Representation, Reasoning and Declarative Problem
  Solving}.
\newblock Cambridge University Press.

\bibitem[\protect\citeauthoryear{Dalal}{Dalal}{1988}]{Dalal88}
{\sc Dalal, M.} 1988.
\newblock Investigations into theory of knowledge base revision.
\newblock In {\em Proceedings of the AAAI National Conference on Artificial
  Intelligence}. St. Paul, Minnesota, 449--479.

\bibitem[\protect\citeauthoryear{Delgrande and Schaub}{Delgrande and
  Schaub}{2003}]{DelgrandeSchaub03}
{\sc Delgrande, J.} {\sc and} {\sc Schaub, T.} 2003.
\newblock A consistency-based approach for belief change.
\newblock {\em Artificial Intelligence\/}~{\em 151,\/}~1-2, 1--41.

\bibitem[\protect\citeauthoryear{Delgrande, Schaub, Tompits, and
  Woltran}{Delgrande et~al\mbox{.}}{2008}]{desctowo08}
{\sc Delgrande, J.}, {\sc Schaub, T.}, {\sc Tompits, H.}, {\sc and} {\sc
  Woltran, S.} 2008.
\newblock Belief revision of logic programs under answer set semantics.
\newblock In {\em Proceedings of the Eleventh International Conference on the
  Principles of Knowledge Representation and Reasoning}, {G.~Brewka} {and}
  {J.~Lang}, Eds. AAAI Press, Sydney, Australia, 411--421.

\bibitem[\protect\citeauthoryear{Eiter, Fink, Sabbatini, and Tompits}{Eiter
  et~al\mbox{.}}{2002}]{Eiteretal02}
{\sc Eiter, T.}, {\sc Fink, M.}, {\sc Sabbatini, G.}, {\sc and} {\sc Tompits,
  H.} 2002.
\newblock On properties of update sequences based on causal rejection.
\newblock {\em Theory and Practice of Logic Programming\/}~{\em 2,\/}~6,
  711--767.

\bibitem[\protect\citeauthoryear{Foo and Zhang}{Foo and
  Zhang}{1997}]{foo-zhan-97a}
{\sc Foo, N.} {\sc and} {\sc Zhang, Y.} 1997.
\newblock Towards generalized rule-based updates.
\newblock In {\em {Proceedings of the Fifteenth International Joint Conference
  on Artificial Intelligence {\rm (}IJCAI'97\/{\rm )}}}. Vol.~1. {Morgan
  Kaufmann}, 82--88.

\bibitem[\protect\citeauthoryear{G{\"a}rdenfors}{G{\"a}rdenfors}{1988}]{Garden%
fors88}
{\sc G{\"a}rdenfors, P.} 1988.
\newblock {\em Knowledge in Flux: Modelling the Dynamics of Epistemic States}.
\newblock The MIT Press, Cambridge, MA.

\bibitem[\protect\citeauthoryear{Gebser, Kaufmann, Neumann, and Schaub}{Gebser
  et~al\mbox{.}}{2007}]{gekanesc07a}
{\sc Gebser, M.}, {\sc Kaufmann, B.}, {\sc Neumann, A.}, {\sc and} {\sc Schaub,
  T.} 2007.
\newblock Conflict-driven answer set solving.
\newblock In {\em Proceedings of the Twentieth International Joint Conference
  on Artificial Intelligence (IJCAI'07)}, {M.~Veloso}, Ed. AAAI Press/The MIT
  Press, 386--392.

\bibitem[\protect\citeauthoryear{Gelfond and Lifschitz}{Gelfond and
  Lifschitz}{1990}]{GelfondLifschitz90}
{\sc Gelfond, M.} {\sc and} {\sc Lifschitz, V.} 1990.
\newblock Logic programs with classical negation.
\newblock In {\em Proceedings of the International Conference on Logic
  Programming}, {D.~H.~D. Warren} {and} {P.~Szeredi}, Eds. MIT, 579--597.

\bibitem[\protect\citeauthoryear{Inoue and Sakama}{Inoue and
  Sakama}{1999}]{inou-saka-99}
{\sc Inoue, K.} {\sc and} {\sc Sakama, C.} 1999.
\newblock Updating extended logic programs through abduction.
\newblock In {\em {Proceedings of the Fifth International Conference on Logic
  Programming and Nonmonotonic Reasoning {\rm (}LPNMR'99\/{\rm )}}}. {Lecture
  Notes in Artificial Intelligence}, vol. 1730. Springer, 147--161.

\bibitem[\protect\citeauthoryear{Leone, Pfeifer, Faber, Eiter, Gottlob, Perri,
  and Scarcello}{Leone et~al\mbox{.}}{2006}]{LeoneEtAl06}
{\sc Leone, N.}, {\sc Pfeifer, G.}, {\sc Faber, W.}, {\sc Eiter, T.}, {\sc
  Gottlob, G.}, {\sc Perri, S.}, {\sc and} {\sc Scarcello, F.} 2006.
\newblock The {DLV} system for knowledge representation and reasoning.
\newblock {\em ACM Transactions on Computational Logic\/}~{\em 7,\/}~3,
  499--562.

\bibitem[\protect\citeauthoryear{Lifschitz, Pearce, and Valverde}{Lifschitz
  et~al\mbox{.}}{2001}]{LifschitzPearceValverde01}
{\sc Lifschitz, V.}, {\sc Pearce, D.}, {\sc and} {\sc Valverde, A.} 2001.
\newblock Strongly equivalent logic programs.
\newblock {\em ACM Transactions on Computational Logic\/}~{\em 2,\/}~4,
  526--541.

\bibitem[\protect\citeauthoryear{Satoh}{Satoh}{1988}]{Satoh88}
{\sc Satoh, K.} 1988.
\newblock Nonmonotonic reasoning by minimal belief revision.
\newblock In {\em Proceedings of the International Conference on Fifth
  Generation Computer Systems}. Tokyo, 455--462.

\bibitem[\protect\citeauthoryear{Simons, Niemel{\"a}, and Soininen}{Simons
  et~al\mbox{.}}{2002}]{SimonsEtAl02}
{\sc Simons, P.}, {\sc Niemel{\"a}, I.}, {\sc and} {\sc Soininen, T.} 2002.
\newblock Extending and implementing the stable model semantics.
\newblock {\em Artificial Intelligence\/}~{\em 138,\/}~1-2, 181--234.

\bibitem[\protect\citeauthoryear{Tasharrofi}{Tasharrofi}{2009}]{Tasharrofi09}
{\sc Tasharrofi, S.} 2009.
\newblock Belief revision through {ASP} programs.
\newblock Unpublished manuscript.

\bibitem[\protect\citeauthoryear{Williams}{Williams}{1995}]{Williams95}
{\sc Williams, M.-A.} 1995.
\newblock Iterated theory base change: A computational model.
\newblock In {\em Proceedings of the International Joint Conference on
  Artificial Intelligence}. Montr\'{e}al, 1541--1547.

\bibitem[\protect\citeauthoryear{Zhang and Foo}{Zhang and
  Foo}{1998}]{DBLP:conf/ecai/ZhangF98}
{\sc Zhang, Y.} {\sc and} {\sc Foo, N.~Y.} 1998.
\newblock Updating logic programs.
\newblock In {\em Proceedings of the Thirteenth European Conference on
  Artificial Intelligence {\rm (}ECAI'98\/{\rm )}}. 403--407.

\end{thebibliography}


\end{document}